\documentclass{article}
\usepackage{color}
\usepackage{graphicx}
\usepackage{amsopn}
\usepackage{fullpage}
\usepackage{amsmath}
\usepackage{amsthm}
\usepackage{commath}
\usepackage{amssymb}
\usepackage{algorithm}
\usepackage{algorithmic}
\usepackage{float}
\usepackage{caption}
\usepackage{subcaption}
\usepackage{xspace}
\usepackage{multirow}

\newtheorem{claim}{Claim}
\newtheorem{theorem}{Theorem}
\newtheorem{assumption}{Assumption}
\newtheorem{lemma}{Lemma}

\DeclareMathOperator{\poly}{poly}
\DeclareMathOperator{\B}{B}
\newcommand\TD{\textsc{Tensor-Decomp}\xspace}
\newcommand\FTD{\textsc{Feature-Tensor-Decomp}\xspace}

\def\P{\mathbb{P}}
\def\E{\mathbb{E}}

\def\R{\mathbb{R}}
\def\bin{\text{Bin}}

\def\diag{\text{diag}}
\def\mi{\text{min}}
\def\bet{\text{beta}}
\def\one{\text{1}}
\def\T{\mathcal{T}}
\def\G{\mathcal{G}}
\def\H{\mathcal{H}}

\usepackage{makecell}

\usepackage{xcolor}
\usepackage{soul}
\usepackage[utf8]{inputenc}
\newcommand\mypara[1]{\medskip\noindent{\textbf{#1}}}

\def\Idea{Modification}

\title{Spectral Learning of Binomial HMMs for DNA Methylation Data}

\author{
Chicheng Zhang$^1$,
Eran A. Mukamel$^2$,
Kamalika Chaudhuri$^3$
\\
$^1$ Microsoft Research \\
$^2$ Dept. of Computer Science and Engineering, University of California San Diego \\
$^3$ Dept. of Cognitive Science, University of California San Diego \\
chicheng.zhang@microsoft.com,
kamalika@cs.ucsd.edu,
emukamel@ucsd.edu
}

\begin{document}

\maketitle

\begin{abstract}
We consider learning parameters of Binomial Hidden Markov Models, which may be used to model DNA methylation data. The standard algorithm for the problem is EM, which is computationally expensive for sequences of the scale of the mammalian genome.  Recently developed spectral algorithms can learn parameters of latent variable models via tensor decomposition, and are highly efficient for large data. However, these methods have only been applied to categorial HMMs, and the main challenge is how to extend them to Binomial HMMs while still retaining computational efficiency. We address this challenge by introducing a new feature-map based approach that exploits specific properties of Binomial HMMs. We provide theoretical performance guarantees for our algorithm and evaluate it on real DNA methylation data.
\end{abstract}

\setlength{\abovecaptionskip}{5pt}
\setlength{\belowcaptionskip}{0pt}
\setlength{\dblfloatsep}{0pt}
\setlength{\dbltextfloatsep}{0pt}

\section{Introduction}

Epigenetic modifications of DNA are critical for a wide range of biological processes. DNA methylation, i.e. the addition of a methyl group to genomic cytosine, is a well known and much studied epigenetic modification of mammalian genomes. However, the locations of methyl marks throughout the genome in particular cell types are largely unknown. Recent advances in genome sequencing technology now make it possible to measure DNA methylation with single-base resolution throughout the genome using whole-genome bisulfite sequencing (BS-Seq) \cite{Mo2015-gq,Schultz2015-gu}.  However, this data are limited by incomplete coverage and stochastic sampling of cells. Probabilistic methods for modeling DNA methylation data sets are therefore essential.

A DNA methylation profile, also called a methylome, is a binary sequence, where a location corresponds to the position of a cytosine base within the genome. A location $t$ has some underlying functional property, modeled by a hidden state $h_t$, and may be methylated with some probability $p_{h_t}$ that depends on the underlying state. For example, locations within the promoter regions of genes generally have very low probability of DNA methylation, while gene bodies and intergenic regions in general are highly methylated. Measuring methylation in a collection of cells of the same type gives a pair of counts $(c_t, \mu_t)$, where the coverage $c_t$ represents how many measurements were made (i.e. how many cytosines sequenced), and $\mu_t$ represents how many of these were methylated. Thus, $\mu_t$ is distributed as $\bin(c_t, p_{h_t})$. A natural framework for modeling this sequence data is therefore a hidden Markov model with binomial distributed observations (Binomial HMM). Given a DNA methylation sequence, our goal is to find the parameters of the underlying model, which describe the states, their methylation probabilities, and also their transitions. These parameters may be further used to segment the genome into compartments, corresponding to hidden states, with potential biologically relevant function\cite{Ernst2012-jr,Roadmap_Epigenomics_Consortium2015-xo}.

The main challenge in addressing this problem is that mammalian genomes are long sequences, with $\sim 10^9$ nucleotides. This means that algorithms such as EM, which are generally used to learn latent variable models, are extremely expensive as each iteration makes a pass over the entire methylation sequence. Recent work by~\cite{AGHKT12} has led to the emergence of spectral algorithms that can learn parameters of latent variable models via tensor decomposition. These algorithms are highly efficient and can operate on large volumes of data and provide theoretical guarantees for convergence to a global optimum. However,  the challenge is that they only apply to categorical HMMs, and unlike likelihood-based methods, extending them to more complex models while still preserving computational efficiency is non-trivial. One plausible approach is to convert Binomial HMMs to categorical HMMs by converting the observations $(c_t, \mu_t)$ to categorical observations. However, this leads to a very high dimensional observation space, which in turn makes the algorithm expensive. A second plausible approach is kernelization~\cite{SADX14}; however, the kernelized algorithm requires a running time quadratic in the sequence length, which again is prohibitively large for real data.

To address this, we observe that none of the proposed approaches (direct conversion or kernelization) exploit specific properties of Binomial HMMs. In particular, an important property is that estimating the joint distribution of $c_t$ and $\mu_t$ given $h_t$ is not necessary, and a direct estimation of $\cbr{p_h}_{h=1}^m$ may result in better statistical efficiency. We exploit this key property to propose a novel feature map-based moditification of the tensor decomposition algorithm of~\cite{AGHKT12}, which we call \FTD. We equip this algorithm with a novel feature map, called the Beta Map, that is tailored to Binomial HMMs. We then provide a novel parameter recovery procedure, which works with the Beta Map to recover the final binomial probability parameters. Finally, we make the entire algorithm more robust against model mismatch by providing a novel stabilization procedure for recovering the transition parameters.

We evaluate our algorithm both analytically and empirically. Theoretically, we provide performance guarantees which show that the proposed algorithm recovers the parameters correctly provided the methylation sequence is long enough, and the granularity of the Beta Map is fine enough with respect to the difference in the methylation probabilities across states. Empirically, we evaluate our algorithm on synthetic and real DNA methylation data. Our experiments show that in all cases our proposed algorithm is an order of magnitude faster than EM, which allows us to run it on much longer sequences. Our algorithm has lower estimation error as well as lower variance on synthetic data, and manages to recover a known pattern of methylation probabilities on real data.



\section{Preliminaries}

\subsection{The Generative Model}

DNA methylation in mammals occurs mainly at CG dinucleotides, which occur roughly once for every 100 bases. Therefore, in this work we model the genome using non-overlapping 100 base-pair bins tiling the genome.  A DNA methylation data set from whole-genome bisulfite sequencing is represented by a sequence of pairs of integers $\cbr{x_t = (c_t, \mu_t)}_{t=1}^l$. At the genomic bin labeled by position $t$, $c_t \in \cbr{0,1,\ldots,N}$ is the {\em coverage}, i.e. number of sequenced DNA fragments that map to CG sites in that bin. The integer $\mu_t \in \cbr{0,1,\ldots,c_t}$ is the {\em methylation count}, which represents the number of DNA fragments that were found to be methylated at those CG sites. On average, $c_t$ is typically about $30$ for high-quality methylome data sets.


We model this data by a binomial hidden Markov model. At position $t$, there is an underlying hidden state $h_t$ in $\{ 1, 2, \ldots, m \}$ generating the observation $x_t$, the dynamics of which are modeled by a Markov chain. A hidden state $h \in \{1, \ldots, m\}$ is associated with a methylation probability $p_h \in [0,1]$. The coverage $c_t$'s are observed. Given $c_t$ and $h_t$, the methylation count $\mu_t$ is drawn from a binomial distribution, with the mean parameter $p = p_{h_t}$. Thus, the binomial HMM model is represented by parameters $(\pi, T, p)$, where $\pi \in \R^m$ is the initial distribution over states, $T \in \R^{m \times m}$ is the transition matrix of the Markov chain over the states, and $p \in \R^m$ is the methylation probability vector.

\subsection{Notation}
For a matrix $M$, we denote by $M^l$ its $l$-th column, and denote by $M_{i,j}$ its $(i,j)$-th entry. Given a third order tensor $\T \in \R^{n_1 \times n_2 \times n_3}$, we use $\T_{i_1, i_2, i_3}$ to denote its $(i_1, i_2, i_3)$-th entry. The tensor product of vectors $v_1$, $v_2$ and $v_3$, denoted by $v_1 \otimes v_2 \otimes v_3$, whose $(i_1, i_2, i_3)$-th entry is $(v_1)_{i_1} (v_2)_{i_2} (v_3)_{i_3}$. A tensor $T$ is called symmetric if $\T_{i_1, i_2, i_3} = \T_{i_{\pi(1)}, i_{\pi(2)}, i_{\pi(3)}}$ for any permutation $\pi: \cbr{1,2,3} \to \cbr{1,2,3}$. Given a tensor $\T \in \R^{n_1 \times n_2 \times n_3}$ and matrices $V_i \in \R^{n_i \times m_i}$, $i = 1,2,3$, $\T(V_1, V_2, V_3)$ is a tensor of size $m_1 \times m_2 \times m_3$, whose $(i_1, i_2, i_3)$-th entry is given by: $\T(V_1, V_2, V_3)_{i_1, i_2, i_3} = \sum_{j_1, j_2, j_3} \T_{j_1, j_2, j_3} (V_1)_{j_1,i_1} (V_2)_{j_2, i_2} (V_3)_{j_3,i_3}$.

 Given a sample $S$ of size $M$ drawn from some distribution, we use $\hat{\E}[\cdot]$ to denote the empirical expectation over $S$, that is, given a function $f$, $\hat \E[f(x)] = \frac 1 M \sum_{x \in S} f(x)$.






\subsection{Background: Spectral Learning for Categorical HMMs}
\label{sec:td}


Our proposed algorithm builds on~\cite{AGHKT12}, who provide a spectral algorithm for parameter estimation in categorial Hidden Markov Models. \cite{AGHKT12} consider categorical HMMs  -- where observations $x_t$ are drawn from a categorical distribution on $\{ 1, \ldots, n \}$. The distribution of $x_t$ when the hidden state is $h$ is given by $O_{h}$, where $O$ is a parameter matrix in $\R^{n \times m}$. As the model is different from the binomial HMM, the results of~\cite{AGHKT12} do not directly apply to our setting; we show in Section~\ref{sec:alg} how they can be adapted with a few novel modifications.



We begin with a brief overview of the algorithm of~\cite{AGHKT12}, abbreviated as \TD. The key idea is that if certain conditions on the parameters of the HMM holds, then, decomposing a tensor of third moments of the observations can give a transformed version of the HMM parameters. The \TD\ algorithm has three steps -- first, it constructs certain cooccurence matrices and tensors; second, it decomposes the cooccurrence tensor after transforming it to ensure tractable decomposition, and finally, the parameters are recovered from the decomposition results.

\begin{itemize}

\item {\bf{Step 1: Construct Cooccurrence Matrices and Tensors.}} First, compute the empirical matrices $\hat{P}_{i,j} := \hat{\E}[x_i \otimes x_j]$ (where $i,j$ are distinct elements from $\cbr{1,2,3}$~\footnote{Although only the first three observations are used, the algorithm can be generalized to use all three consecutive observations in the sequences.}) and tensor $\hat{\T} := \hat{\E}[x_1 \otimes x_2 \otimes x_3]$.

\item{\bf{Step 2: Transform and Decompose.}} The tensor $\hat \T$ is related to the parameter $O$, which can in theory be recovered by decomposing $\hat \T$. However, directly decomposing $\hat \T$ is computationally intractable~\cite{HL13}, and so the tensor is converted to symmetric orthogonal form.

To symmetrize $\hat \T$, we compute the matrices $S_1 := \hat P_{2,3} \hat P_{1,3}^\dagger$ and $S_3 := \hat P_{2,1} \hat P_{3,1}^\dagger$. Then, compute $\G := \hat \T(S_1, I, S_3)$, which is symmetric. To orthogonalize $\G$, compute a matrix $J := S_3 \hat P_{3,2}$; take an SVD of $J$ and take the top $m$ singular vectors $U_m$, and singular values in diagnoal matrix $S_m$, getting the orthogonalization matrix $W = U_m S_m^{-1/2}$. Perform a linear transformation over tensor $\G$ using $W$, getting the symmetric orthogonal tensor $\H = \G(W, W, W)$.

$\H$ is then decomposed via the Tensor Power Method~\cite{AGHKT12} to recover the eigenvectors $v_l$'s and the eigenvalues $\lambda_l$.

\item {\bf{Step 3: Recover Parameters.}} The columns of $O$ can now be recovered as $\hat{O}^l = (W^T)^{\dagger} \hat\lambda_l \hat{v}_l$, for $l=1,2,\ldots,m$. Compute an estimate of the joint probablity of $h_2$ and $h_1$: $\hat{H}_{21}:=\hat{O}^{T\dagger} \hat P_{21} \hat{O}^\dagger$. Then estimate the initial probability vector and the transition matrix by $\hat\pi := \one^T \hat{H}_{21}^T$ and $\hat{T} := \hat{H}_{21} \diag(\hat \pi)^{-1}$.

\end{itemize}

Observe that \TD has two advantages over popular parameter estimation algorithms such as EM. First, it only needs to make one pass over the data, and is thus computationally efficient if the size of the observation space is not too large. In contrast, EM proceeds iteratively and needs to make one pass over the data per iteration. Second, it achieves statistical consistency if the data is generated from an HMM (see Theorem 5.1 of~\cite{AGHKT12}), in contrast with methods such as EM, which do not have statistical consistency guarantees.

\section{Algorithm}
\label{sec:alg}

The main limitation of spectral algorithms is that unlike likelihood-based methods, they cannot be readily extended to more general models. In particular, the \TD\ algorithm for categorial HMMs does not directly apply to our setting. One plausible approach is to convert our model to a categorical HMM by converting the observations $(c_t, \mu_t)$ to categorical observations. However, this leads to a very high dimensional observation space, leading to high computational cost. A second plausible approach is kernelization~\cite{SADX14}; however, the kernelized algorithm requires a running time quadratic in the sequence length, which again is prohibitively large for real data.


To address this challenge, we observe that neither direct conversion nor kernelization exploits specific properties of Binomial HMMs. In particular, an important property is that estimating the joint distribution of $c_t$ and $\mu_t$ given $h_t$ is not necessary, and a direct estimation of $\cbr{p_h}_{h=1}^m$ may result in better statistical efficiency.

We exploit this key property to propose a novel feature map-based moditification of \TD, which we call \FTD. We equip \FTD\ with a novel feature map, called the Beta Map, that is tailored to Binomial HMMs. We then provide a novel parameter recovery procedure, which works with the Beta Map to recover the final binomial probability parameters. Finally, we make the entire algorithm more robust against model mismatch by providing a novel stabilization procedure for recovering the transition parameters.



\subsection{Key Components}


We next describe the key modifications that we make to \TD\ to adapt it to binomial HMMs.

\paragraph{\Idea\ 1: Feature Map.} Our first contribution is to use a feature map $\phi$ to map our discrete observations $x = (c, \mu)$ to a $D$-dimensional vector $\phi(x)$. Thus, instead of computing the cooccurrence matrices and tensors in Step 1 of \TD, we now compute the empirical {\em{feature co-occurrence matrices and tensors}}:
\[ \hat P_{i,j} := \hat{\E}[\phi(x_i) \otimes \phi(x_j)], \quad \hat \T := \hat{\E}[\phi(x_1) \otimes \phi(x_2) \otimes \phi(x_3)],\]
and apply the remaining steps of \TD on these matrices and tensors. We call the resulting algorithm \FTD.

What does \FTD recover? Define the matrix $C$ as a $D \times m$ matrix whose $j$-th row is: $C^j = \E[\phi(x)|h = j]$; this is analogous to the observation matrix $O$ in~\cite{AGHKT12}. Provided certain conditions hold, we show in Section~\ref{sec:theory} that \FTD\ can, given sufficiently many sequences, provably recover a high quality estimate of $C$.


\paragraph{\Idea\ 2: Beta Mapping.} We next propose a novel feature map, called the Beta Map, that is tailored to Binomial HMMs; the map is inspired by the popular Beta-Binomial conjugate prior-posterior system in Bayesian inference.

Define the function $\varphi_{\bet, D}((c,\mu), t) = \frac{1}{\B(\mu+1, c-\mu+1)} t^\mu (1-t)^{c-\mu}$ as the density of the Beta distribution with shape parameters $\mu+1$ and $c-\mu+1$, where $\B(\cdot,\cdot)$ is the Beta function. Now, given an observation $(c, \mu)$, the Beta Map $\phi_{\bet, D}(c, \mu)$ is an $n$-dimensional vector whose $i$-th entry is:
\begin{equation}
  (\phi_{\bet, D}(c, \mu))_i := \int_{(i-1)/D}^{i/D} \varphi_{\bet, D}((c,\mu), t) dt
\label{eqn:fm}
\end{equation}
We apply Algorithm $\FTD$ with the feature map $\phi_{\bet, D}$.

The Beta Map has two highly desirable properties. First, it maps the observation $(\mu_t, c_t)$ to a probability mass function with mean close to $\frac{\mu_t}{c_t}$, which is $p_{h_t}$ in expectation. Second, if $c_t$ is large, then the feature map is highly concentrated around $\approx \frac{\mu_t}{c_t}$, reflecting higher confidence in the estimated $p_{h_t}$.

\paragraph{\Idea\ 3: Recovery of the Binomial Probabilities.} As we observe earlier, running $\FTD$ with the Beta Map will recover the expected feature map matrix $C$, and not the binomial probabilities. We now provide a novel recovery procedure to estimate $p_h$ from $C$.



Suppose $\phi = \phi_{\bet,D}$ is the Beta mapping, and $\hat{C}$ is the estimated feature map recovered by \FTD. Observe that for any $h$, the expected Beta map $C^h$ is a mixture of Beta distributions, where each mixture component corresponds to a pair $(c, \mu)$ and has shape parameter $(\mu+1,c-\mu+1)$ (and hence mean $\approx \frac{\mu+1}{c+2}$) and mixing weight $p(c_t=c,\mu_t=\mu|h_t=h)$. As the mean of the mixture distribution is equal to the weighted average over the component means, this gives us: $\frac{1}{D} \sum_{i=1}^D \frac{i}{D} (C)_{i,h} \approx \E[\frac{\mu+1}{c+2}|h]$. Assuming that $h$ and $c$ are independent, the right hand side simplifies to $a + (1-2a) p_h$, where $a := \E[\frac{1}{c+2}]$. Using $\hat a := \hat \E[\frac{1}{c+2}]$ instead of $a$, this gives the following recovery equation~\eqref{eqn:recoverp} for $p_h$:
\begin{equation}
  \hat{p}_h := \frac{\frac{1}{D} \sum_{i=1}^D \frac{i}{D} (\hat{C})_{i,h} - \hat a}{1 - 2 \hat a}
  \label{eqn:recoverp}
\end{equation}

Due to the estimation error in $C$ and the discretization of the Beta mapping, we cannot hope to recover $p_h$ exactly.
However, if $C$ is accurately recovered, and the granularity of the Beta map is fine, then our estimate of $p_h$ is accurate.
Our estimation grows accurate with increasing granularity of discretization, at the expense of a higher running time.


\paragraph{\Idea\ 4: Stabilization.} Model mismatch in real data often leads to unstable solutions in~\cite{AGHKT12}, especially in recovering the transition matrix $T$. To prevent instability, we propose an alternative approach: a least squares formulation to recover $\pi$ and $T$. Given $\hat{C}$ as an estimatior of $\E[\phi(x)|h]$, we propose to solve the following optimization problem:
\[ \min_{H_{2,1}: \forall i,j (H_{2,1})_{i,j} \geq 0, \sum_{i,j} (H_{2,1})_{i,j} = 1} \| P_{2,1} - \hat{C} H_{2,1} \hat{C}^T \|_F^2 \]
Here, $(H_{2,1})_{i,j}$ is our proposed estimator of $\P(x_2 = i, x_1 = j)$, $i,j \in \cbr{1,\ldots,m}$. Next, we recover the transition matrix and intial probability by applying the formulae $\hat{\pi} =\one^T H_{2,1}$ and $\hat{T} = H_{2,1} \diag(\hat{\pi})^{-1}$. The key difference betwen this procedure and the Step 3 of \TD is that, our optimization problem ensures that our estimators of $\pi$ and $T$ are entrywise positive, and thus no postprocessing are needed for subsequent usage of the parameters.

\subsection{Extension: Multiple Cell Types}
Finally, an additional goal for us is to study multiple aligned methylation sequences in order to identify differential methylation states, where the expected methylation probabilities are different across different cell types.

Here, we observe two coverage methylation pairs per location, one for each cell type, so an observation $x = ((c^1, \mu^1), (c^2, \mu^2))$. Our goal is to estimate a pair of methylation probabilities $(p^1_h, p^2_h)$ for each state $h$ that is shared across cells. To this end, we construct a concatenated feature map:
\[ \Phi(x) = \begin{bmatrix} \phi(c^1, \mu^1) \\ \phi(c^2, \mu^2) \end{bmatrix} \]
Following the tensor decomposition algorithm in Section 2, we can recover the expected feature map
given hidden states:
\[ C = \E[\Phi(x)|h] = \begin{bmatrix} \E[\phi(c^1, \mu^1)|h] \\ \E[\phi(c^2, \mu^2)|h] \end{bmatrix}\]
Applying the recovery procedure of $p_h$ to each block gives a pair $(p^1_h, p^2_h)$ for each hidden state $h$. If we see a large difference between $p^1_h$ (methylation probability in cell type 1) and $p^2_h$ (methylation probability in cell type 2), then we identify state $h$ as a differential methylation state.


\section{Performance Guarantees}
\label{sec:theory}


The \TD\ algorithm has provable guarantees when an underlying condition, called the Rank Condition, on the parameters of the HMM holds. For \FTD, the analogous condition is the Feature Rank Condition below.

\begin{assumption}[Feature Rank Condition]
The expected feature map matrix $C \in \R^{D \times m}$ and the transition matrix $T \in \R^{m \times m}$ are of full column rank.
\label{assumption:rcfm}
\end{assumption}

The Feature Rank Condition is satisfied for Beta Maps when the $p_h$ values are well separated and the granularity of the Beta Map is high. Formally, if $q$ is the minimum gap $\min_{i \neq j} |p_i - p_j|$, then, Theorem~\ref{thm:rcfm} shows that as long as the discretization parameter $D$ and the coverage $c$ are above some function of $q$, the Feature Rank Condition is satisfied. For simplicity we assume here that the coverage $c$ is fixed.


\begin{theorem}
Suppose $D \geq \frac{4}{q}$ and $c \geq \frac{512}{q^2}$. Consider the Beta feature map $\phi_{\bet, D}(x)$ defined as in Equation~\eqref{eqn:fm}. Then $C$, the expect feature map matrix, has minimum singular value at least
$\frac{1}{2\sqrt{D}}$, and is thus of full column rank.
\label{thm:rcfm}
\end{theorem}


Provided the conditions of Theorem~\ref{thm:rcfm} hold, we can show statistical consistency of \FTD; the proof is given in the appendix.


\begin{theorem}[Statistical consistency of \FTD]
Suppose \FTD receives $M$ iid samples drawn from a binomial hidden Markov model represented by parameters $(\pi, T, p_c, p)$. In addition, suppose $T$ is of full rank, $\pi$ is positive entrywise, and the coverage $c$ is large enough.
Then with high probability, the $\ell_2$ distances between the outputs $\hat{p}$, $\hat{T}$ and $\hat{\pi}$ and the respective underlying parameters $\pi, T, p$ converge to zero, with increasing sample size $M$ and Beta Map dimension $D$.
\label{thm:mpc}
\end{theorem}



\section{Experiments}







 Our evaluation of the empirical performance of \FTD\ has two major goals. First, we aim to validate our theoretical results by examining how it performs against EM when data is truly generated from a Binomial HMM. Real data typically has model mismatch, and our second goal is to investigate how \FTD\ performs under model mismatch by comparing it with EM on real DNA methylation data.

\subsection{Validation on Synthetic Data}



We begin with simulations on synthetic data with known underlying generative parameters.


\mypara{Data Generation.} We generate a single sequence binomial HMM with four hidden states. For each state, we generate a methylation probability $p_h$. To ensure there is a gap in the probabilities, $p_1$ and $p_2$ are uniformly drawn from $[0,0.3]$, and $p_3$ and $p_4$ from $[0.7,1]$. We generate a transition matrix $T$ as: $T = 0.2 I_4 + 0.8 U$, where all elements of matrix $U$ are drawn independently and uniformly from $[0,1]$, and then each column is normalized to sum to 1. We generate the initial probability vector $\pi$ by normalizing a random vector $u$, with entries drawn uniformly and independently from $[0,1]$. We consider three coverage settings -- the coverage $c_t$ drawn from a Poisson with means $25$ (low coverage), $50$ (medium coverage) and $100$ (high coverage). For each set of parameters, we draw $8$ sequences of size from $\cbr{2^7, \ldots, 2^{13}}$.





\mypara{Methodology.} We compare $\FTD$ with EM. Both algorithms have a hyperparameter -- the number of states $m$ that we set to the correct value $4$. For \FTD\ the number of tensor power iterations per component is $2$ and the granularity for Beta Map is $30$. EM is stopped at iteration $t$ if the fractional decrease in the log-likelihood is below $0.001$.

Recall that neither \FTD\ nor EM produces the hidden states in correct order, and hence to evaluate the estimation error, we need to match the estimated states with the true ones. We do so by using the Hungarian algorithm to find the minimum cost matching between states where the cost of matching a state $p$ to $\hat{p}$ is the difference $|p - \hat{p}|$. The estimation error is defined as the cost of the best cost matching. We repeat the experiment for $20$ trials, and plot the mean and the standard error of the estimation error in Figure~\ref{fig:synthetic-error}. The running times are reported in Table~\ref{tab:runtime}.



\mypara{Results.}
From Figure~\ref{fig:synthetic-error}, we observe that for small training sample size, EM and \FTD have comparable estimation error, whereas for large sample size ($>$ 1000), \FTD performs substantially better.
In general, we find that \FTD is robust as our theory predicts, while the results of EM depend strongly on initialization.
 Table~\ref{tab:runtime} illustrates that \FTD is an order of magnitude faster than EM for all sample sizes, which means that it can be indeed be run on larger datasets. We remark that the average running time is not strictly monotonic increasing with sample size, which may be attributed to the fact that we use caching of the feature maps to speed up the implementation.
 These observations indicate that \FTD should be preferred when large amounts of data are available, while EM may be used for smaller data problems. An additional observation is that the estimation error under EM typically has higher variance than \FTD\ for larger sample sizes. Finally, the degree of coverage does not appear to make a big difference to the results.



%

\begin{figure}[t]
  \centering
    \begin{tabular}{c}
        \begin{subfigure}[t]{0.3\textwidth}
          \includegraphics[width=\textwidth]{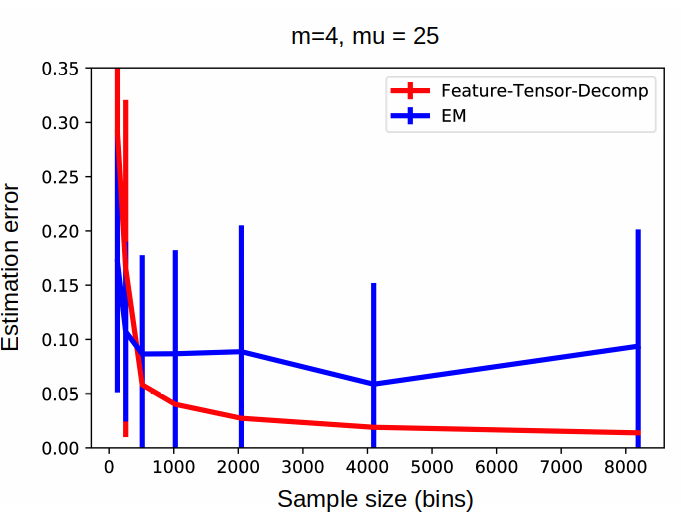}
        \end{subfigure}
        \\
        \begin{subfigure}[t]{0.3\textwidth}
          \includegraphics[width=\textwidth]{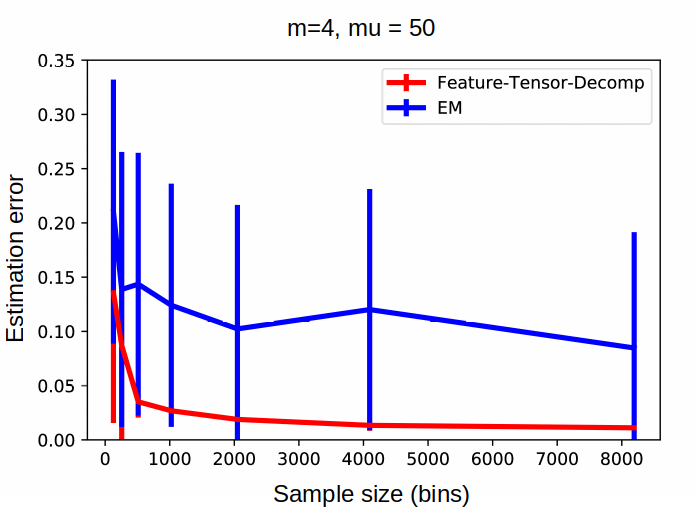}
        \end{subfigure}
        \\
        \begin{subfigure}[t]{0.3\textwidth}
          \includegraphics[width=\textwidth]{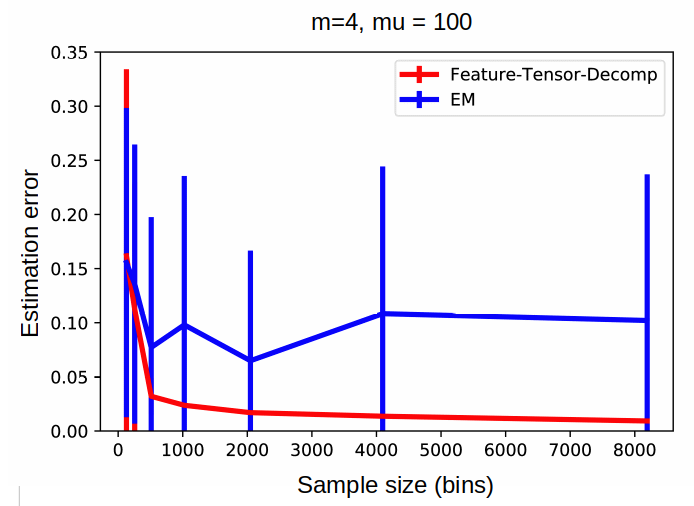}
        \end{subfigure}
      \end{tabular}
      \caption{Estimation error vs sample size for \FTD (red) and EM (blue). From top to bottom: three settings with varying mean coverage parameters: 25, 50 and 100. All results are averaged over 20 repeated trials; the error bars
      represent the standard deviation of the estimation error value over these trials.}
      \label{fig:synthetic-error}
\end{figure}

\begin{table}
  \centering
\begin{tabular}{|lll|}
  \hline
   Sample Size $\backslash$ Algorithm & FTD & EM \\
   \hline
   128 & 0.6758 & 5.838 \\
   256 & 0.9948 & 5.718 \\
   512 & 1.4427 & 7.713 \\
   1024 & 1.8173 & 17.468 \\
   2048 & 1.4253 & 28.105 \\
   4096 & 2.7181 & 62.430 \\
   8192 & 4.2342 & 82.218 \\
   \hline
\end{tabular}
\caption{Average running time (in seconds) for EM and \FTD with varying training sample sizes. All results are averaged over 20 repeated trials. We only present the result when the average coverage is 25; the cases when the average coverage are 50, 100 are similar.}
\label{tab:runtime}
\end{table}



\subsection{Real Data Experiments}

We next investigate how \FTD\ performs on real DNA methylation data.


\mypara{Data.} We use the DNA methylation dataset from~\cite{Mo2015-gq}, and use chromosome 1 of the mouse
 genome. The data has a pair of (coverage, methylation) counts for every bin consisting of $100$ DNA base pairs.  We select two cell types -- excitatory neurons (E) and VIP cells (V) -- which are known to have differential methylation. The data has two replicate sequences for each cell type; we merge them by adding the coverage and methylation counts at each location. The data also has contextual information for each position that is derived from the underlying DNA sequence. We are specifically interested in the CG context that is of biological relevance. We extract the subsequence that is restricted to the CG context, which gives us a sequence of length 1923719.




\mypara{Methodology.} For real data, we compare three algorithms -- \FTD (FTD), EM and \FTD\ followed by $3$ rounds of EM (FTD+EM). The motivation for studying the last algorithm is that it has been previously reported to have good performance on real data~\cite{Zhang2014,Chaganty2013}. We set the number of hidden states to $m=6$. As running EM until convergence on this long sequence is slow, we run it for $10$ iterations and report the results.

Evaluation on real data is more challenging since we do not have the ground truth parameters. We get around this by using the following two measures. First, we compute the log-likelihood of the estimated parameters over a separate test set. Second, we look at the cell types E and V which are known to have differential methylation regions in the CG context, and verify that this is reflected in our results. Finally, for aligned sequences from multiple cell types, the observation probabilities given the hidden states are computed under a conditional independence assumption.



\subsection{Results}

Figure~\ref{fig:realll} plots the log-likelihood on a test set under EM, FTD and FTD+EM. The running times are reported in Table~\ref{tab:real}, and the estimated methylation probability matrices are plotted in Figure~\ref{fig:methylationmat}.

From Figure~\ref{fig:realll}, we see that EM achieves slightly better log-likelihood than FTD and FTD+EM. This is not surprising since EM directly maximizes the likelihood unlike FTD. We see from Table~\ref{tab:real} that for the sample size of $4 \times 10^4$, FTD is two orders of magnitude faster than that of EM. The bottleneck in FTD+EM's running time is the 3 rounds of EM. In addition, FTD processes almost the whole sequence (of length $1.6 \times 10^6$) in a relatively short amount of time ($<$280s), which is faster than EM on a sample of size $4 \times 10^4$ ($>$1175s).

In Figure~\ref{fig:methylationmat}, we permute the columns of each estimated methylation probability matrix to make them as aligned as possible. We mark the aligned states using vertical lines between the columns. We can draw a few conclusions from the estimated matrices: first, all algorithms identify a hidden state with extremely low methylation level on both cells (column 1, $<$0.025) and a hidden state with extremely high methylation level on both cells (column 2, $>$0.91). Second, all algorithms identify a hidden state with low methylation on cell E and high methylation on cell V (column 6), although the difference is most stark for \FTD\ run on a sample size of $1.6 \times 10^6$. Third, the EM based algorithms (FTD and FTD+EM) find a state that has relatively high methylation on cell E and low methylation on cell V (column 5), whereas FTD does not identify this state. FTD+EM appears to achieve the best of both worlds -- relatively low running time with correctly identified states.


\begin{table}
  \centering
\begin{tabular}{|ll|}
\hline
Algorithm (sample size) & Running Time (s) \\
\hline
FTD($4 \times 10^4$) &  11.608894 \\
FTD+ EM($4 \times 10^4$) & 357.817675 \\
EM($4 \times 10^4$) &  1125.509846 \\
FTD($1.6\times10^6$) &  274.1687 \\
\hline
\end{tabular}
\caption{Running Time for a set of algorithm and sample size pairs.}
\label{tab:real}
\end{table}


\begin{figure}
\centering
\includegraphics[scale=0.45]{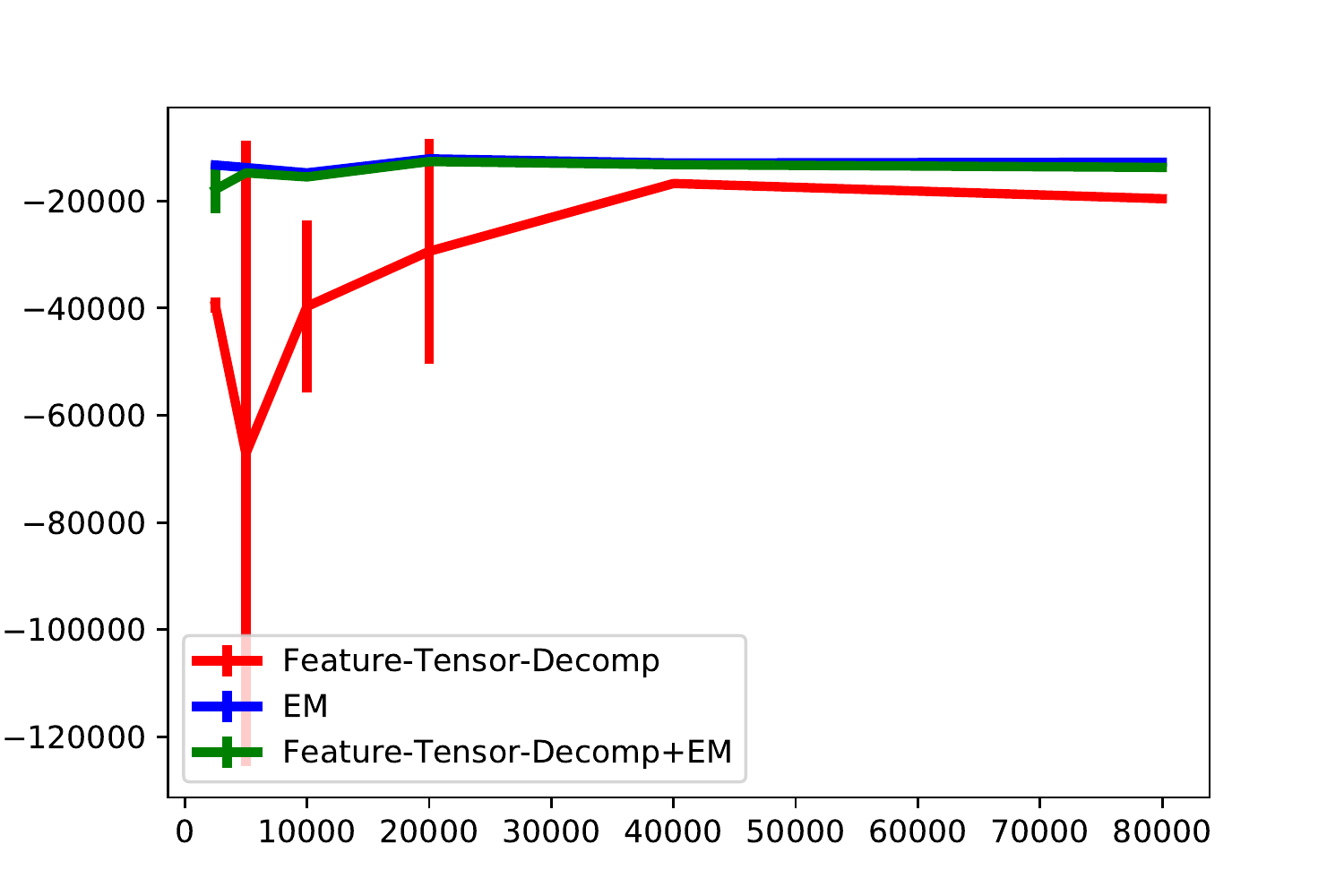}
\caption{Test log-likelihood vs training sample Size for algorithms \FTD (red), EM (blue), and EM with \FTD initialization (green).}
\label{fig:realll}
\end{figure}

\begin{figure}
  \centering
  \scalebox{0.8}
  {
\begin{tabular}{rccccccl}
  \multirow{2}{*}{ \Big[} & 0.023 & 0.953 & 0.229  & 0.859 & 0.709 &  0.583 &\multirow{2}{*}{ \Big] EM($4\times10^4$)} \\
  & 0.023 & 0.954 & 0.476   & 0.781 & 0.457 & 0.903 & \\
  & $\vert$ & $\vert$ & & & & $\vert$ \\
  \multirow{2}{*}{ \Big[} & 0.023 & 0.955 & 0.749  & 0.472 & 0.858 & 0.168 & \multirow{2}{*}{ \Big] FTD + EM($4\times10^4$)}\\
  & 0.026 & 0.948 & 0.827  & 0.379 & 0.419 & 0.811 & \\
  & $\vert$ & $\vert$ & & & & $\vert$ \\
  \multirow{2}{*}{ \Big[} & 0.01 & 0.916 & 0.298 & 0.345 & 0.649 & 0.206 & \multirow{2}{*}{ \Big] FTD($4\times10^4$)}\\
  & 0.01 & 0.925 & 0.257 & 0.450 & 0.880 & 0.535 & \\
  & $\vert$ & $\vert$ & & & & $\vert$ \\
  \multirow{2}{*}{ \Big[} & 0.01 & 0.99 & 0.598 & 0.838 & 0.950 & 0.132 & \multirow{2}{*}{ \Big] FTD($1.6\times10^6$)} \\
  & 0.01 & 0.99 & 0.519 & 0.990 & 0.913 & 0.752 & \\

\end{tabular}
}
\caption{The methylation probability matrices (of size $2 \times 6$) for each algorithm on datasets of different sizes. Row 1: Cell E, Row 2: Cell V.}
\label{fig:methylationmat}
\end{figure}





\section{Related Work}



There has been a large body of work on hidden Markov models with structured
emission distributions~\cite{Rabiner1989}. The binomial hidden Markov model has
been proposed in \cite{couvreur1996}, and is used in \cite{giampieri2005} for
financial applications.

A line of work has developed spectral learning algorithms for latent variable
models, including HMMs. ~\cite{AGHKT12} provides an elegant algorithm for
categorical HMMs based on moment matching and tensor decomposition, which
provably recovers the parameters when the training data is generated from an
HMM.  The algorithm provides no guarantees under model mismatch, and
works only for categorical HMMs, which has a finite output space. \cite{Song2015}
applies this algorithm to Chip-Seq data to recover underlying states.
\cite{SADX14} proposes kernelized versions of spectral methods that generalize
the algorithm of~\cite{AGHKT12} to handle data with rich observation spaces;
however, a direct application of \cite{SADX14}'s algorithm has running time quadratic
in the sample size, thus making it prohibitive for methylation data.
\cite{zhang2015spectral} studies spectral learning for a
structured HMM with multiple cell types, namely HMM with tree hidden states
(HMM-THS)~\cite{Biesinger2013}. Their model is different from ours, in that our
HMM has only one state controlling the observations at each time, whereas
HMM-THS has one state per sequence.  \cite{Zou2013} studies the setting where
there are two sequences, foreground and background, and the goal is to find
states that appear exclusively in the foreground sequence.  Their problem
setting is different from ours, in that the roles of the two sequences are
asymmetric, unlike ours.
Finally, it has been observed both empirically and theoretically in the
spectral learning literature that EM initialized with the output of the
spectral methods achieve better accuracy in many tasks, such as
crowdsourcing~\cite{Zhang2014} and in learning a mixture of linear
regressors~\cite{Chaganty2013}.



\subsection*{Acknowledgments}
We thank NSF under IIS-1617157 for research support.
EAM is supported by NIH/NINDS (NS080911).
\bibliographystyle{plain}
\bibliography{reference}

\newpage
\appendix
\section{Sample Complexity - Proof of Theorem~\ref{thm:mpc}}


We present Theorem~\ref{thm:mpcformal} below, which immediately implies Theorem~\ref{thm:mpc}.
To see this, suppose the coverage $c$ is $\geq \frac{512}{q^2}$, and we are given
given learning parameters $\epsilon, \delta \in (0,1)$.
Now, set $D_0 = \max(\frac 4 q, \frac {16} \epsilon)$. By Theorem~\ref{thm:rcfm},
for all $D \geq D_0$, $\sigma_\mi(C) \geq \frac{1}{2\sqrt{D}}$.
Given such $D$, by Theorem~\ref{thm:mpcformal}, we can find a value of $N_0 = \poly( \frac{1}{\min_i \pi_i}, D, \frac{1}{\sigma_\mi(T)}, \frac{1}{\epsilon}, \ln \frac{1}{\delta})$,
 such that for every
$N \geq N_0$, with probability $1-\delta$,
the distances between $\hat{p}$, $\hat{T}$ and $\hat{\pi}$ and the respective underlying parameters $\pi, T, p$ are all at most $\epsilon$ in terms of the respective error metrics.

\begin{theorem}[Statistical Consistency of \FTD]
Suppose \FTD receives $M$ iid samples $(x_1, x_2, x_3)$ as input, which are drawn from a binomial hidden Markov model represented by parameters $(\pi, T, p_c, p)$. In addition, suppose the distribution on coverage $c$ is such that $a := \E[\frac 1 {c+2}] \leq \frac 3 8$.
Then, given learning parameters $\epsilon$ and $\delta$ in $(0,\frac 1 {16})$, if
the Beta feature map dimension $D$ is $\frac{16}{\epsilon}$, and
the number of samples $M$ is at least $\poly( \frac{1}{\min_i \pi_i}, \frac{1}{\sigma_\mi(C)}, \frac{1}{\sigma_\mi(T)}, \frac{1}{\epsilon}, \ln \frac{1}{\delta}, D)$,
then with probability $1-\delta$, the output $\hat{p}$, $\hat{T}$ and $\hat{\pi}$ satisfies that
\[ \| p - \hat{p} \Pi \|_2 \leq \epsilon \]
\[ \| T - \Pi^\dagger \hat{T} \Pi \|_F \leq \epsilon \]
\[ \| \pi - \Pi^\dagger \hat  \pi \|_2 \leq \epsilon \]
for some permutation matrix $\Pi$. Here $\sigma_\mi(M)$ denotes the minimum singular value of matrix $M$.
\label{thm:mpcformal}
\end{theorem}

\subsection{Recovering Initial Probability, Transition Matrix and Expected Feature Map}
To prove Theorem~\ref{thm:mpc}, we will apply the sample complexity bounds of \TD for HMM~\cite{AGHKT12}. Specifically, we will apply a result implicit in~\cite{AGHKT12}, which appears explicitly in~\cite{zhang2015spectral}.

\begin{theorem}[Initial Probability, Transition Matrix and Expected Feature Map Consistency]
Suppose \FTD receives $m$ iid samples $(x_1, x_2, x_3)$ as input, which are drawn from a categorical hidden Markov model represented by parameters $(\pi, T, O)$. Then, given parameters $\epsilon$ and $\delta$ in $(0,1)$, if the number of samples $M$ is at least $\poly( \frac{1}{\min_i \pi_i}, \frac{1}{\sigma_\mi(C)}, \frac{1}{\sigma_\mi(T)}, \frac{1}{\epsilon}, \ln \frac{1}{\delta})$,
then with probability $1-\delta$, \FTD produces estimated expected feature map $\hat{C}$, transition matrix $\hat{T}$ and initial probability $\hat{\pi}$ such that
\[ \| C - \hat{C} \Pi \|_F \leq \epsilon \]
\[ \| T - \Pi^\dagger \hat{T} \Pi \|_F \leq \epsilon \]
\[ \| \pi - \Pi^\dagger \hat  \pi \|_2 \leq \epsilon \]
for some permutation matrix $\Pi$. Here $\sigma_\mi(M)$ denotes the minimum singular value of a matrix $M$.
\label{thm:efmc}
\end{theorem}
\begin{proof}
The proof is almost the same as the proof of (\cite{zhang2015spectral}, Theorem 1), by taking $V$ as a set of size $D = 1$. The only difference between our proof and theirs is the argument for the concentration of the raw moments, which we address in Lemma~\ref{lem:rawconc} below.
\end{proof}

\begin{lemma}
Suppose we are given $M$ iid triples $(x_{i,1}, x_{i,2}, x_{i,3})$, $i \in \cbr{1,2,\ldots,M}$.
Then with probability $1-\delta$, the following concentration inequalities hold simultaneously:
\[ \| P_{12} - \hat{P}_{12} \|_F \leq \epsilon(M,\delta) \]
\[ \| P_{23} - \hat{P}_{23} \|_F \leq \epsilon(M,\delta) \]
\[ \| P_{13} - \hat{P}_{13} \|_F \leq \epsilon(M,\delta) \]
\[ \| \T - \hat{\T} \|_F \leq \epsilon(M,\delta) \]
where $\epsilon(M, \delta) = \sqrt{\frac{4 + 4\ln(8/\delta)}{M}}$.
\label{lem:rawconc}
\end{lemma}
\begin{proof}
The lemma is a direct consequence of Lemma~\ref{lem:l1conc} below by taking $\xi_i$ as the vectorization of  $\phi(x_{i,1}) \otimes \phi(x_{i,2})$, $\phi(x_{i,2}) \otimes \phi(x_{i,3})$, $\phi(x_{i,1}) \otimes \phi(x_{i,3})$
and $\phi(x_{i,1}) \otimes \phi(x_{i,2}) \otimes \phi(x_{i,3})$ respectively, along with a union bound.
\end{proof}

\begin{lemma}
Suppose we are given a sequence of $M$ iid $d$-dimensional vectors $\xi_i$, $i=1,2,\ldots,M$, and $\|\xi_i\|_1 \leq 1$ almost surely. In addition, denote by $\Xi$ the expectation of the $\xi_i$'s, and denote by $\hat{\xi}:=\frac1M \sum_{i=1}^M \xi_i$ the empirical mean of the $\xi_i$'s. Then, with probability $1-\delta$,
\[ \| \hat{\xi} - \Xi \|_2 \leq \sqrt{\frac{4+4\ln\frac{2}{\delta}}{M}} \]
\label{lem:l1conc}
\end{lemma}
\begin{proof}
The lemma follows from the ideas in (\cite{HsuKZ2012}, Appendix A). We first show that $\E \| \hat{\xi}  - \Xi \|_2 \leq \frac 2 M$. We justify it as follows:
\[ \E \| \hat \xi - \Xi \|_2^2 = \E \| \sum_{i=1}^M  \frac1M (\xi_i - \Xi) \|_2^2 = \frac1{M^2} \sum_{i=1}^M \E\| \xi_i - \Xi \|_2^2 \leq \frac 4 {M} \]
where the last inequality is due to that $\| \xi_i \|_2 \leq \| \xi_i \|_1 \leq 1$, and $\|\Xi\|_2 \leq \E \| \xi_i \|_2 \leq 1$.
Consequently, $\E \| \hat{\xi}  - \Xi \|_2 \leq \sqrt{\E \| \hat{\xi}  - \Xi \|_2^2} \leq \sqrt{\frac 4 M}$.

Next, note that replacing a $\xi_i$ with a $\xi_i'$ changes the function $\| \hat{\xi} - \Xi \|_2$ by at most $\frac 2 M$. Applying McDiarmid's inequality, we get that with probability $1-\delta$,
\[ | \| \hat{\xi}  - \Xi \|_2 - \E \| \hat{\xi}  - \Xi \|_2 | \leq \sqrt{\frac{4\ln\frac{2}{\delta}}{M}}\]
This implies that
\[ \| \hat{\xi} - \Xi \|_2 \leq  \E \| \hat{\xi}  - \Xi \|_2 + \sqrt{\frac{4\ln\frac{2}{\delta}}{M}} \leq \sqrt{\frac{4+4\ln\frac{2}{\delta}}{M}}. \]
\end{proof}

\subsection{From Expected Feature Map to Binomial Probability}
We have shown in the above section that the expected feature map, transition matrix and initial probability vector can be accurately recovered. However, the recovery accuracy of the binomial probability $p$ still remains unaddressed. In this section, we address this issue.
Specifically, we establish Lemma~\ref{lem:D-large}, which shows the following implication: if we get an accurate enough recovery of $C_2$ and we have a large enough $D$ (the discretization granularity of our Beta feature map), then we can get an accurate recovery of $p$. Recall that $\hat a := \hat \E[\frac{1}{c+2}]$ is the empirical mean of $a := \E[\frac{1}{c+2}]$.


\begin{lemma}
\label{lem:D-large}
Suppose that the hidden state $h$ and the coverage $c$ are independent. For any $\epsilon \in (0,\frac 1 {16})$, if $D \geq \frac {16} \epsilon$, and for two columns $h$ and $h'$, $\| \hat{C}^{h' }- C^h \|_2 \leq \frac \epsilon {16 \sqrt{D}}$, $| \hat{a} - a | \leq \frac \epsilon {256}$ and $a \leq \frac 3 8$, then
\[ | \hat{p}_{h' } - p_h | \leq \epsilon. \]
\end{lemma}

Before going into the proof of Lemma~\ref{lem:D-large}, we need a lemma that characterize the recovery of $p$ in the setting of infinite $D$ and infinite sample size.

\begin{lemma}
  The binomial probability $p_h$ can be written in terms of the Beta density as follows:
  \begin{equation}
    p_h = \frac{\int_0^1 t\E[\varphi_{\bet}(x,t) | h] dt - a}{1 - 2 a}.
    \label{eqn:idealp}
  \end{equation}
  \label{lem:ph}
\end{lemma}
\begin{proof}
Recall that from standard calculus,
\begin{eqnarray*}
  &&\int_0^1 \frac{t \cdot t^\mu (1-t)^{c-\mu}}{\B(\mu+1,c-\mu+1)} dt \\
  &=&
  \int_0^1 \frac{t \cdot t^{(\mu+1)-1} (1-t)^{(c-\mu+1)-1}}{\B(\mu+1,c-\mu+1)} dt = \frac{\mu+1}{c+2}.
\end{eqnarray*}
Therefore,
\begin{eqnarray*}
  \int_0^1 t \E[\varphi_{\bet}(x,t) | h] dt = \E\sbr{\left. \frac{\mu+1}{c+2} \right\vert h} & \\
  = \E\sbr{\left. \frac{cp_h + 1}{c + 2}  \right\vert h} = \E\sbr{\frac{c}{c+2}} p_h + \E\sbr{\frac{1}{c+2}}, &
\end{eqnarray*}
where the last step uses the independence assumption between $h$ and $c$.
Recall that $a = \E[\frac{1}{c+2}]$, the above can be rewritten as:
\[ \int_0^1 t\E[\varphi_\bet(x,t) | h] dt = (1 - 2a) p_h + a. \]
The equality in the lemma statement follows by algebra.
\end{proof}


\begin{proof}[Proof of Lemma~\ref{lem:D-large}]

Without loss of generality, suppose $h$ and $h'$ are the same. 

Denote $\hat{\gamma}_{D,h}$ (resp. $\gamma_{D,h}$, $\gamma_h$) by $\frac{1}{D} \sum_{i=1}^D \frac{i}{D} \hat{C}_{h,i}$
(resp. $\frac{1}{D} \sum_{i=1}^D \frac{i}{D} C_{h,i}$, $\int_0^1 t\E[\varphi_\bet(x,t) | h] dt$).
Observe that $\hat{\gamma}_{D,h}$, $\gamma_{D,h}$, and $\gamma_h$ are all in $[0,1]$.
Using this notation, the recovery formula for $p_h$ in \FTD (Equation~\eqref{eqn:recoverp}) can be written as
\[ \hat{p}_h = \frac{\hat{\gamma}_{D,h} - \hat{a}}{1 - 2 \hat{a}}. \]
In addition, Lemma~\ref{lem:ph} implies that $p_h$ can be written as
\[ p_h = \frac{\gamma_h - a}{1 - 2 a}. \]
Note that by triangle inequality, $| \hat{\gamma}_{D,h} - \gamma_{D,h} | \leq \| C^{h} - \hat{C}^{h} \|_1 \leq \frac \epsilon {16\sqrt D} \sqrt{D} = \frac \epsilon {16}$.
In addition, we can bound $| \gamma_h - \gamma_{D,h} |$ as follows:
\begin{eqnarray*}
  && | \gamma_h - \gamma_{D,h} | \\
  &=& | \sum_{i=1}^D \frac{i}{D} \E[(\phi_{\bet}(x))_i | h] - \int_0^1 t\E[\varphi_\bet(x,t) | h] dt | \\
  &=& | \sum_{i=1}^D \frac{i}{D} \int_{\frac {i-1} n}^{\frac i n} \E[\varphi_\bet(x,t) | h] dt -\\
   && \sum_{i=1}^n \int_{\frac {i-1} D}^{\frac i D} t\E[\varphi_\bet(x,t) | h] dt |
  \\
  &\leq& \sum_{i=1}^D | \int_{\frac {i-1} D}^{\frac i D} (\frac{i}{D} - t) \E[\varphi_\bet(x,t) | h] dt |
  \\
  &\leq& \sum_{i=1}^D \frac 1 D \int_{\frac {i-1} D}^{\frac i D} \E[\varphi_\bet(x,t) | h] dt \\
  &\leq& \frac 1 D \leq \frac \epsilon {16}
\end{eqnarray*}
Combining the above two facts, we have $| \hat{\gamma}_h - \gamma_{h} | \leq | \hat{\gamma}_h - \gamma_{D,h} | + | \gamma_h - \gamma_{D,h} | \leq \frac \epsilon 8$.

On the other hand, denote $\xi$ (resp. $\hat{\xi}$) by $\frac{1}{1-2a}$ (resp. $\frac{1}{1-2\hat{a}}$). Note that since  $a \leq \frac 3 8$, $\xi \leq 4$. Thus,
\[ | \xi - \hat{\xi} | = \frac{2|a - \hat{a}|}{(1-2\hat{a})(1-2a)} \leq 64 |a - \hat{a}| \leq \frac \epsilon 4 \]
Therefore,
\begin{eqnarray*}
   && | p_h - \hat{p}_h | \\
   &=& | (\hat{\xi} - \xi)(\hat{\gamma}_h - \gamma_h) + \xi (\hat{\gamma}_h - \gamma_h) + \gamma_h (\hat{\xi} - \xi)| \\
    &\leq& | (\hat{\xi} - \xi)(\hat{\gamma}_h - \gamma_h)| + |\xi (\hat{\gamma}_h - \gamma_h)| + |\gamma_h (\hat{\xi} - \xi)| \\
   &\leq& \frac \epsilon 4 \frac \epsilon 8 + 4 \frac \epsilon 8 + \frac \epsilon 4 \\
   &\leq& \epsilon.
\end{eqnarray*}
This completes the proof.
\end{proof}

\subsection{Putting It Together}
Built on Theorem~\ref{thm:efmc} and Lemma~\ref{lem:D-large}, we are ready to prove Theorem~\ref{thm:mpcformal}.
\begin{proof}[Proof of Theorem~\ref{thm:mpcformal}]
First, given the choice of $D \geq \frac {16} \epsilon$, if we have a sample of size $\poly( \frac{1}{\min_i \pi_i}, \frac{1}{\sigma_\mi(C)}, \frac{1}{\sigma_\mi(T)}, \frac{1}{\epsilon}, \ln \frac{1}{\delta}, D)$,
by Theorem~\ref{thm:efmc}, with probability $1-\delta/2$,
 produces estimated expected feature map $\hat{C}$, transition matrix $\hat{T}$ and initial probability $\hat{\pi}$ such that
 \begin{equation}
   \| C - \hat{C} \Pi \|_2 \leq \frac \epsilon {16\sqrt D},
   \label{eqn:c}
 \end{equation}
\[ \| T - \Pi^\dagger \hat{T} \Pi \|_F \leq \frac \epsilon {16\sqrt D} \leq \epsilon, \]
\[ \| \pi - \Pi^\dagger \hat  \pi \|_2 \leq \frac \epsilon {16\sqrt D} \leq \epsilon. \]
for some permutation matrix $\Pi$.
In addition, if the sample size $M$ is at least $O(\frac{\ln \frac 1 \delta}{\epsilon^2})$, by Hoeffding's inequality, with probability $1-\delta/2$,
\begin{equation}
  |\hat{a} - a| \leq \frac \epsilon {256}
  \label{eqn:a}
\end{equation}
Denote by $E$ the intersection of the above two events.
By union bound, the probabilty of $E$ happening is at least $1-\delta$. Conditioned on event $E$,
the recovery accuracy of the transition matrix and the initial probability are satisfied. We now argue that the recovery accuracy of the binomial probability is also satisfied.

Denote by $\pi$ the permutation induced by $\Pi$, i.e. for a $m$-dimensional row vector $v$, $v \Pi = (v_{\pi(1)}, \ldots, v_{\pi(n)})$. Then equation~\eqref{eqn:c} implies that for all $i \in \cbr{1,\ldots,m}$,
\[ \| C^i - \hat{C}^{\pi(i)} \|_2 \leq \frac \epsilon {16\sqrt{D}}. \]
Now, applying Lemma~\ref{lem:D-large}, we get that for all $i \in \cbr{1,\ldots,m}$,
\[ | p_i - \hat{p}_{\pi(i)} | \leq \epsilon. \]
written in the matrix form,
 \[ \| p - \hat{p} \Pi \|_2 \leq \epsilon. \]
\end{proof}


\section{Proof of Theorem \ref{thm:rcfm}}

In this section, we give the proof of Theorem~\ref{thm:rcfm}, showing that if the coverage is large enough, the binomial probabilities are sufficiently separated, and the discretization parameter $D$ is large enough, then the feature rank condition is well-satisfied.

\begin{proof}[Proof of Theorem~\ref{thm:rcfm}]
Define vector $C^i$ as the $i$th column of the expected feature map matrix $C$. We note that for each $i$, $C^i$ is the average of $\phi_{\bet,D}(c,\mu)$ for different $(c,\mu)$'s.
In addition, for every $(c,\mu)$, $\phi_{\bet,D}(c,\mu)$ is a probability vector. This implies that
for every $i$, $C^i$ is a probabilty vector.
To show that $C$ is of full column rank, we show that each column $C_i$ is approximately supported on disjoint coordinates. Formally we have the following claim.


\begin{claim}
There exist sets $S_i \in \cbr{1,2,\ldots,D}$ for each $i \in \cbr{1,2,\ldots,m}$, such that:
\begin{enumerate}
  \item $S_i$'s are disjoint, that is, for all distinct $i,j$, $S_i \cap S_j = \emptyset$.
  \item Each $C^i$ is well-supported on $S_i$: $\sum_{j \in S_i} C_{i,j} \geq \frac 3 4$.
\end{enumerate}
\end{claim}
\begin{proof}
Fix a column $j \in \cbr{1,\ldots,m}$. Define a random variable $\mu$ drawn from $\bin(c, p_j)$, and conditioned on $\mu$,  a random variable $G$ is drawn from the Beta distribution with shape parameters $\mu+1$ and $c-\mu+1$. Recall that $\phi_{\bet,D}(c, \mu)_i$ is defined as $\int_{(i-1)/D}^{i/D} \varphi_{\bet, D}((c,\mu), t) dt$, and is thus equal to $\P[ G \in (\frac{i-1}{D}, \frac{i}{D}] | \mu]$.
By the law of total expectation,
\begin{eqnarray*}
  C_{i,j} &=& \E[\phi_{\bet,D}(x)_i|h=j] \\
  &=& \E[\E[\phi_{\bet,D}(x)_i|\mu]|h=j]  \\
  &=& \E[\P[G \in (\frac{i-1}{D}, \frac{i}{D}]|\mu]|h=j] \\
  &=& \P[ G \in (\frac{i-1}{D}, \frac{i}{D}]].
\end{eqnarray*}

Now, for each $j \in \cbr{1,2\ldots,m}$, define set $S_j = \cbr{ i \in \cbr{1,\ldots,D}: i \in [D( p_j - 4\sqrt{\frac{\ln 4} c}) - 1, D(p_j + 4\sqrt{\frac{\ln 4}c}) + 1] }$. Under this setting
of the $S_j$'s, we now prove
the two items in the lemma statement respectively.



\begin{itemize}
  \item We now show the first item. Consider two distinct indices $j, l \in \cbr{1,2,\ldots,m}$. Without loss of generality, suppose $p_l < p_j$.
We show that the maximum element of $S_l$ is strictly smaller than the minimum element of $S_j$, thus establishing
the disjointness of the two sets. To this end, we show that
\[ D( p_j - 4\sqrt{\frac{\ln m} c}) - 1 > D( p_l + 4\sqrt{\frac{\ln m} c}) + 1. \]
The reason is as follows:
since $D \geq \frac{4}{q} \geq \frac{4}{p_j - p_l}$, $\frac{2}{D} \leq \frac{p_j - p_l}{2}$. In addition,
since $c \geq \frac{512}{q^2} > \frac{256 \ln 4}{(p_j - p_l)^2}$, $8\sqrt{\frac{\ln 4}c} < \frac{p_j - p_l}{2}$. Therefore, the above inequality holds.

 \item For the second item, we apply concentration inequalities on Beta and binomial distributions.
First, by Hoeffding's Inequality, with probability $\frac{7}{8}$,
\[ \mu \in [cp_j - \sqrt{c \ln 4} , cp_j - \sqrt{c \ln 4}] \]
Second, given $\mu$, $G$ is Beta distributed with shape parameters $\mu$ and $c-\mu$, and is thus $\frac{1}{4c}$-sub-Gaussian~\cite{MA17}. We have that with probability $\frac{7}{8}$,
\[ c G \in [ \mu - \sqrt{c \ln 4}, \mu - \sqrt{c \ln 4} ] \]
Therefore, by union bound and algebra, we have that with probability $\frac 3 4$,
\[ G \in [p_j - 2\sqrt{\frac{\ln 4}{c}}, p_j + 2\sqrt{\frac{\ln 4}{c}}]\]
This gives that
\begin{eqnarray*}
  \sum_{i \in S_j} C_{i,j} &=& \P[G \in \cup_{i \in S_j} (\frac{i-1}{D}, \frac{i}{D}] ]  \\
  &\geq& \P[ G \in [p_j - 2\sqrt{\frac{\ln 4}{c}}, p_j + 2\sqrt{\frac{\ln 4}{c}}] ] \\
  &\geq& \frac 3 4.
\end{eqnarray*}
where the first inequality is by the fact that
$[p_j - 2\sqrt{\frac{\ln 4}{c}}, p_j + 2\sqrt{\frac{\ln 4}{c}}]$ is a subset of $\cup_{i \in S_j} (\frac{i-1}{D}, \frac{i}{D}]$.
\end{itemize}

\end{proof}

Provided the above claim holds, we now lower bound the minimum singular value of $C$. It suffices to show that, for every vector $x$ in $\R^m$, $\| C x \|_1 \geq \frac 1 2  \| x \|_1$.
Indeed, if the above is true, then $\| C x \|_2 \geq \frac 1 {\sqrt{D}} \| C x \|_1 \geq \frac 1 {2\sqrt{D}} \| x \|_1 \geq \frac 1 {2\sqrt{D}} \| x \|_2$, implying that the minimum singular value of $C$ is at least $\frac 1 {2\sqrt{D}}$.

Consider vector $x \in \R^m$. Let matrix $T$ be such that
\[ T_{i,j} = \begin{cases} C_{i,j} & i \in S_j \\ 0 & i \notin S_j \end{cases}\]
and let matrix $S$ to be $C - T$.
Note that as the $S_j$'s are disjoint,
\[ \| Tx \|_1 = \sum_j \sum_{i \in S_j} |C_{i,j} x_j| \geq \frac 3 4 \| x \|_1. \]
Also,
\[ \| Sx \|_1 \leq \sum_j \sum_{i \notin S_j} |C_{i,j} x_j| \leq  \sum_j \frac 1 4 |x_j| \leq \frac 1 4 \| x \|_1. \]
Thus,
\begin{eqnarray*}
  \| Cx \|_1 &=&  \| T x + S x \|_1 \\
  & \geq&  \| T x \|_1 - \| S x \|_1 \geq \frac 1 2  \| x \|_1.
\end{eqnarray*}
The lemma follows.
\end{proof}

\end{document}